\documentclass[letterpaper, 10 pt, conference]{ieeeconf}  

\IEEEoverridecommandlockouts                              

\usepackage{graphicx}
\usepackage{subcaption}
\usepackage{tikz}
\usepackage{latexsym}

\usepackage{epstopdf}
\usepackage{multirow}
\usepackage{mathrsfs}
\usepackage{mathtools}
\usepackage{stmaryrd}
\usepackage{times}
\usepackage{soul}
\usepackage{cite}
\usepackage{url}
\usepackage[hidelinks]{hyperref}
\usepackage[utf8]{inputenc}
\usepackage{bm}

\usepackage{amsthm}
\usepackage{amsthm}
\usepackage{amsfonts}
\usepackage{amssymb}
\usepackage{booktabs}
\usepackage{algorithm}
\usepackage[switch]{lineno}
\usepackage{color}

\usepackage{algpseudocode}
\algnewcommand{\Initialize}[1]{%
  \State \textbf{Initialize:}
  \Statex \hspace*{\algorithmicindent}\parbox[t]{.8\linewidth}{\raggedright #1}
}
\algnewcommand{\Inputs}[1]{%
  \State \textbf{Inputs:}
  \Statex \hspace*{\algorithmicindent}\parbox[t]{.8\linewidth}{\raggedright #1}
}

\usepackage{amsmath,nccmath}
\usepackage{amssymb}
\usepackage{xspace}
\usepackage{xcolor}

\newcommand{\A}{\mathcal{A}} 

\newcommand{\C}{\mathcal{C}}

\newcommand{\F}{\mathcal{F}}
\newcommand{\G}{\mathcal{G}} 
\renewcommand{\H}{\mathcal{H}}

\renewcommand{\L}{\mathcal{L}}
\newcommand{\M}{\mathcal{M}}

\renewcommand{\S}{\mathcal{S}} 
\newcommand{\T}{\mathcal{T}}

\newcommand{\trace}{\pi}

\newcommand{\LTL}{{LTL}\xspace}
\newcommand{\LTLf}{{\sc ltl}$_f$\xspace}

\newcommand{\ls}{{\sf{lst}}}
\newcommand{\post}{{\sf{Post}}}

\newcommand{\alphabet}{\Sigma}
\newcommand{\prop}{\mathtt{Prop}} 
\newcommand{\aut}{\mathcal{A}}
\newcommand{\trans}{\delta}
\newcommand{\acc}{\mathtt{Acc}}
\newcommand{\infset}{\mathtt{Inf}}
\newcommand{\lang}{\mathtt{L}}
\newcommand{\ltltrue}{\mathsf{true}}

\newcommand{\ltlU}{\mathsf{U}}
\newcommand{\ltlF}{\lozenge}
\newcommand{\ltlG}{\square}

\newcommand{\prob}{\rm{Pr}}
\newcommand{\mdpsttr}{\mathcal{F}}
\newcommand{\probtr}{\mathcal{T}}

\newcommand{\fpaths}{\rm{FPaths}}

\newcommand{\distr}{\mathsf{Distr}}

\algnewcommand{\Try}[1]{%
  \State \textbf{Try:}
  \Statex \hspace*{\algorithmicindent}\parbox[t]{.8\linewidth}{\raggedright #1}
}

\algdef{SE}[SWITCH]{Switch}{EndSwitch}[1]{\algorithmicswitch\ #1\ \algorithmicdo}{\algorithmicend\ \algorithmicswitch}%
\algdef{SE}[CASE]{Case}{EndCase}[1]{\algorithmiccase\ #1}{\algorithmicend\ \algorithmiccase}%
\algtext*{EndSwitch}%
\algtext*{EndCase}%

\newtheorem{definition}{Definition}
\newtheorem{lemma}{Lemma}

\newtheorem{problem}{Problem}

\newtheorem{example}{Example}
\newtheorem{theorem}{Theorem}

\title{\LARGE \bf Planning with Linear Temporal Logic Specifications:\\ Handling Quantifiable and Unquantifiable Uncertainty}

\author{Pian Yu$^*$, Yong Li, David Parker, and Marta Kwiatkowska
\thanks{This work was supported by the ERC AdG FUN2MODEL (Grant agreement ID: 834115), ISCAS Basic Research (Grant Nos. ISCAS-JCZD-202406, ISCAS-JCZD-202302), CAS Project for Young Scientists in Basic Research (Grant No. YSBR-040), and ISCAS New Cultivation Project ISCAS-PYFX-202201. }
\thanks{Pian Yu is currently with the Department of Computer Science, University College London (UCL), United Kingdom
        {\tt\small {pian.yu}@ucl.ac.uk} $^*$Work was done when she was with the Department of Computer Science, University of Oxford.}
\thanks{David Parker and Marta Kwiatkowska are with the Department of Computer Science, University of Oxford, United Kingdom
        {\tt\small {david.parker, marta.kwiatkowska}@cs.ox.ac.uk}}%
\thanks{Yong Li is with Key Laboratory of System Software (Chinese Academy of Sciences) and State Key Laboratory of Computer Science, Institute of Software, Chinese
Academy of Sciences, China
        {\tt\small {liyong}@ios.ac.cn}}
}

\begin{document}

\maketitle
\thispagestyle{empty}
\pagestyle{empty}

\begin{abstract}
This work studies the planning problem for robotic systems under both quantifiable and unquantifiable uncertainty. The objective is to enable the robotic systems to optimally fulfill high-level tasks specified by Linear Temporal Logic (LTL) formulas. To capture both types of uncertainty in a unified modelling framework, we utilise Markov Decision Processes with Set-valued Transitions (MDPSTs).
We introduce a novel solution technique for the optimal robust strategy synthesis of MDPSTs with LTL specifications. To improve efficiency, our work leverages limit-deterministic B\"uchi automata (LDBAs) as the automaton representation for LTL to take advantage of their efficient constructions. To tackle the inherent nondeterminism in MDPSTs, which presents a significant challenge for reducing the LTL planning problem to a reachability problem, we introduce the concept of a Winning Region (WR) for MDPSTs. Additionally, we propose an algorithm for computing the WR over the product of the MDPST and the LDBA. Finally, a robust value iteration algorithm is invoked  to solve the reachability problem. We validate the effectiveness of our approach through a case study involving a mobile robot operating in the hexagonal world, demonstrating promising efficiency gains.  
\end{abstract}


\section{Introduction}

Uncertainty in planning can be categorised into two types based on the effects of actions: probabilistic and nondeterministic. In probabilistic planning, uncertainty is quantified using probabilities, with Markov Decision Processes (MDPs) and their generalisations serving as the standard modelling frameworks \cite{puterman2014markov,natarajan2022planning}. Nondeterministic planning, on the other hand, addresses unquantifiable uncertainty (such as ambiguity and adversarial environments), typically exploiting the fully observable nondeterministic domain (FOND) as a modelling framework \cite{CimattiPRT03,FOND,GeffnerBonet2013}. Both probabilistic and nondeterministic planning have been extensively studied, leading to significant advances in the field.

Robotic systems are susceptible to many different types of uncertainty, such as sensing and actuation noise, unpredictability in a robot's perception, and dynamic environments \cite{thrun2005probabilistic,du2011robot}. Some sources of uncertainty, such as sensing and actuation noise, can be quantified probabilistically using statistical methods. However, ambiguous uncertainties, such as unpredictable perception and dynamic environments, are often more challenging to quantify. Existing works in robot planning mainly focus on addressing either quantifiable or unquantifiable uncertainty. However, in scenarios such as human-robot collaboration \cite{ajoudani2018progress}, both quantifiable and unquantifiable uncertainties are present. Quantifiable uncertainties may arise from robotic actuation errors, while unquantifiable uncertainties often stem from the unpredictable nature of human behavior.
To the best of our knowledge, approaches that can effectively handle both types of uncertainty simultaneously have, however, been less explored.
In light of this, we propose to utilise
MDPs with set-valued transitions (MDPSTs)~\cite{trevizan2007planning,trevizan2008mixed} as our unified modelling framework. 
They are attractive because they admit a simplified Bellman equation compared to (more general) MDPs with imprecise probabilities~(MDPIPs) \cite{white1994markov,satia1973markovian,givan2000bounded} and Uncertain MDPs~(UMDPs)~\cite{nilim2005robust,buffet2005robust,hahn2019interval}, and thus stochastic games~\cite{Con93}. 

Recently, MDPSTs have been employed to formalise the \emph{trembling-hand problem in nondeterministic domains} ~\cite{yu2024trembling}, where the term ``trembling-hand" refers to the phenomenon in which an agent, due to faults or imprecision in its action
selection mechanism, may mistakenly perform unintended actions with a certain probability, potentially leading to goal failures. Specifically, this approach demonstrates that the human-robot co-assembly problem can be modelled using MDPSTs, yielding more efficient solution techniques compared to the stochastic game formulation. In this work, Linear Temporal Logic on finite traces~(\LTLf)~\cite{de2013linear} was used as the task specification language. \LTLf\ shares the same syntax as Linear Temporal Logic (\LTL)~\cite{Pnu77} but is interpreted over \emph{finite} rather than infinite traces~\cite{de2013linear}.
However, in many robotic applications, such as persistent surveillance and repetitive supply delivery, it is necessary to define the robot's tasks over \emph{infinite} trajectories.
Indeed, \LTL has been widely applied in robotics research for specifying complex temporal objectives over infinite traces, such as~\cite{ding2014ltl,schillinger2019hierarchical,guo2018probabilistic,sadigh2014learning,wen2016probably,hasanbeig2019reinforcement}, including 
MDPs with LTL objectives, 
e.g., \cite{kwiatkowska2013automated, kwiatkowska2022probabilistic,BrafmanGP18}. A typical approach involves converting the LTL specification into a Deterministic Rabin Automaton (DRA) and then taking the product of the MDP and the DRA \cite{baier2008principles}. This reduces the problem to a planning problem with a reachability goal over the product space. 

In this paper, we propose to formalise the planning problem for robotic systems under both quantifiable and unquantifiable uncertainty with temporal objectives as the strategy synthesis problem for MDPSTs with (full)  \LTL objectives.
We highlight that MDPSTs with LTL objectives are 
studied for the first time in this paper. Due to the presence of unquantifiable uncertainty in MDPSTs, computing the end components \cite{de1998formal} of an MDPST becomes nontrivial. As a result, the conventional procedure for MDPs with LTL objectives  based on conversion to DRAs does not apply, necessitating new solution techniques developed in this work.

The main contributions are summarised as follows. (i) We propose using MDPSTs as the modelling framework for robot planning under both quantifiable and unquantifiable uncertainty. Although the model has been relatively little studied since it was initially proposed in 2007 \cite{trevizan2007planning}, recent findings highlight its advantages, particularly in terms of computational efficiency \cite{yu2024trembling}. (ii) A novel solution technique is proposed for the optimal robust strategy synthesis for MDPSTs with LTL specifications. This technique addresses the inherent nondeterminism in MDPSTs, which complicates the reduction of the LTL planning problem to a reachability problem, by introducing the concept of a Winning Region (WR). 
To further improve efficiency, we leverage limit-deterministic B\"uchi automata (LDBAs)~\cite{courcoubetis1995complexity,hahn2013lazy,sickert2016limit}, which are typically smaller than conventional DRAs thanks to their efficient construction from LTL. 
We devise an algorithm for computing
WR over the product of the MDPST and the LDBA, and its correctness is demonstrated formally.

\section{Preliminaries}

This section provides preliminaries for \LTL \cite{Pnu77} and its equivalent LDBA  \cite{sickert2016limit} representation. 


\LTL~\cite{Pnu77} extends propositional logic with temporal operators.
The syntax of an \LTL formula over a finite set of propositions $\prop$ is defined inductively as:
\begin{equation}\label{LTL}  \varphi::= \ltltrue | p \in \prop |\neg\varphi|\varphi \wedge\varphi |\varphi \vee\varphi |\bigcirc \varphi|\varphi \mathsf{U}\varphi,
\end{equation}
where $\bigcirc$ (Next) and $\ltlU$ (Until) are temporal operators.
As usual, additional Boolean 
and temporal operators are derived as follows: $\varphi_1 \Rightarrow \varphi_2 \equiv\neg \varphi_1 \vee \varphi_2$ (Implies), $\lozenge \varphi \equiv \ltltrue \ltlU\varphi$ (Eventually), and $\ltlG \varphi = \neg (\ltlF \neg\varphi)$ (Always). The detailed semantics of \LTL can be found in \cite{Pnu77,baier2008principles}.

A \textit{trace} $\trace = \trace_0\trace_1\ldots$ is a finite or infinite sequence of propositional interpretations~(sets), where for every $i \geq 0$, $\trace_i \in 2^{\prop}$ is the $i$-th interpretation in $\trace$. Intuitively, $\trace_i$ is interpreted as the set of propositions that are $\ltltrue$ at instant~$i$.
For a finite trace $\trace \in (2^{\prop})^{*}$, we denote the interpretation at the last instant~(i.e., index) by $\ls(\trace)$,
and we write $\trace \models \varphi$ when an infinite trace $\trace \in (2^{\prop})^{\omega}$ satisfies \LTL formula~$\varphi$.
The language of $\varphi$, denoted
by $\lang(\varphi)$, is the set of infinite traces over $2^{\prop}$ that satisfy $\varphi$.

Every \LTL formula $\varphi$ over $\prop$ can be translated into a \emph{nondeterministic B\"{u}chi automaton} (NBA) $\aut$ \cite{DBLP:journals/iandc/VardiW94} over the alphabet $\alphabet = 2^{\prop}$ that recognises the language $ \lang(\varphi)$.
\begin{definition}[]\label{def:nba}
    An NBA $\aut$ is defined as a tuple $\aut = (Q, \alphabet, q_0, \trans, \acc)$, where $Q$ is the set of states, $q_0$ is the initial state, $\acc \subseteq Q$ is the set of accepting states, and $\trans : Q \times \alphabet \mapsto 2^Q$ is the nondeterministic transition function.
\end{definition}
A \emph{run} $\rho $ of $\aut$ over an infinite trace $w_0 w_1\cdots \in \alphabet^{\omega}$ is an infinite sequence $ r_0 r_1 \cdots \in Q^{\omega}$ of states such that $r_0 = q_0$ and, for all $i\geq 0$, we have $r_{i+1} \in \trans(r_i, w_i)$.
We denote by $\infset(\rho)$ the set of states that appear infinitely often in the run $\rho$.
A run $\rho$ of $\aut$ is called \emph{accepting} if $\infset(r)\cap \acc \neq \emptyset$.
The language of $\aut$, denoted $\lang(\aut)$, is the set of all traces that have an accepting run in $\aut$.

NBAs, in general, cannot be used for quantitative analysis of probabilistic systems.
Recently, a class of NBAs called limit-deterministic B\"uchi automata (LDBAs), under mild constraints, have been applied for the quantitative analysis of MDPs~\cite{courcoubetis1995complexity,hahn2013lazy,sickert2016limit}.
We will also use the LDBAs constructed by \cite{sickert2016limit} for our planning problem.

\begin{definition}[LDBA \cite{sickert2016limit}]\label{def:ldba}
    An LDBA $\aut$ is defined as a tuple $\aut = (Q, \alphabet, q_0, \trans, \acc)$ where 
\begin{itemize}
  \item $Q=Q_i\cup Q_{acc}$ is the set of states partitioned into two disjoint sets $Q_i$ and $Q_{acc}$,
  \item $q_0 \in Q_i$ is the initial state,
  \item $\acc \subseteq Q_{acc}$ is the set of accepting states, and
  \item $\trans = \trans_i \cup \trans_j \cup \trans_{\acc}$ where $\trans_i : Q_i \times \alphabet \mapsto Q_i, \trans_{acc} : Q_{acc} \times \alphabet \mapsto Q_{acc}$ and $\trans_j : Q_i \times \{\epsilon\} \mapsto 2^{Q_{acc}}$. 
\end{itemize}
\end{definition}

By Definition~\ref{def:ldba}, the LDBAs considered here are deterministic within $Q_i$ and $Q_{acc}$ components; the only nondeterminism lies in the $\epsilon$-transition jumps from $Q_i$-states to $Q_{acc}$-states via $\trans_j$ function.
Note that the $\epsilon$-transitions do not consume a letter from $\alphabet$: they are just explicit representations of the nondeterministic jumps in the runs of $\aut$.
To be accepting in $\A$, a run has to eventually make a nondeterministic jump through $\trans_j$ since all accepting states reside only in $Q_{acc}$.
It is easy to translate an \LTL formula $\varphi$ to an LDBA $\aut$ such that $\lang(\aut) = \lang(\varphi)$ using state-of-the-art tools such as Owl \cite{sickert2016limit} and Rabinizer 4 \cite{kvretinsky2018rabinizer}.

\section{Markov
decision processes with set-valued transitions}

This section introduces Markov
Decision Processes with Set-valued Transitions (MDPSTs) \cite{trevizan2007planning, trevizan2008mixed} as the modelling framework for robot planning under quantifiable and unquantifiable uncertainty.  Compared to the definition in  \cite{trevizan2007planning}, we further introduce a labelling function that associates the states of the MDPST with the propositions of an LTL formula.

\begin{definition}[MDPSTs]\label{def:MDPST}
A \emph{MDPST} $\mathcal{M}$ is a tuple $(S, s_0, A, \mdpsttr, \probtr, \L)$, where 
\begin{itemize}
    \item $S$ is a finite set of states;
    \item $s_0\in S$ is the initial state;
    \item $A$ is a finite set of actions;
    \item $\mdpsttr: S \times A \Mapsto 2^{2^S}$ is the set-valued nondeterministic state transition (partial) function;
    \item $\probtr: S \times A \times 2^S \mapsto (0, 1]$ is the transition probability (or mass assignment) function, i.e., given a set $\Theta\in \mdpsttr(s, a)$, $\probtr(s, a, \Theta)=\prob(\Theta|s, a)$,
    \item $\L: S \to 2^{\prop}$ is the proposition labelling function, where $\prop$ is a finite set of propositions.
\end{itemize}
\end{definition}
Traditional MDPs are, in fact, a special type of MDPSTs, where $\probtr$ only maps a state and an action to a probabilistic distribution over $S$ instead of the powerset $2^{S}$.
As usual, we use $A(s) \subseteq A$ to denote the set of actions \emph{applicable} at state $s$. Note that, in MDPSTs, the transition function $\mdpsttr(s, a)$ returns a set of state sets, i.e., $\F(s,a) \subseteq 2^\S$, and the transition probability function $\probtr$ expresses the probability of transitioning to such sets via a given action. 

A path $\xi$ of $\M$ is a finite or infinite sequence of alternating states and actions $\xi=s_0 a_0 s_1 a_1\cdots$, ending with a state if finite, such that for all $i 
\geq 0$, $a_i \in A(s_i)$ and $s_{i+1}\in \Theta_i$ for some set $\Theta_i \in \F(s_i, a_i)$.
We denote by $\fpaths$ ($\fpaths_s$) and ${\rm IPaths}$ (${\rm IPaths}_s$) the set of all finite and infinite paths of $\M$ (starting from state $s$), respectively. 
For a path $\xi=s_0 a_0 s_1 a_1\cdots$ of $\M$, the sequence $\L(\xi)=\L(s_0)\L(s_1), \cdots$ over $\prop$ is called the \emph{trace} induced by $\xi$ over $\M$.

A strategy $\sigma$ of $\M$ is a function $\sigma: \fpaths \to \distr(A)$ such that, for each $\xi\in \fpaths$, $\sigma(\xi)\in \distr(A({\ls}(\xi)))$, where ${\ls}(\xi)$ is the last state of the finite path $\xi$ and $\distr(A)$ denotes the set of all possible distributions over $A$.
Let $\Omega_\sigma^{\M}(s)$ denote the subset of (in)finite paths of $\M$ that correspond to strategy $\sigma$ and initial state $s$.
Let $\Pi_{\M}$ be the set of all strategies.

Let us now motivate MDPSTs for robot planning under uncertainty with a running example.

\vspace{-0.2cm}
\begin{figure}[ht]
\centering
\includegraphics[width=0.4\textwidth]{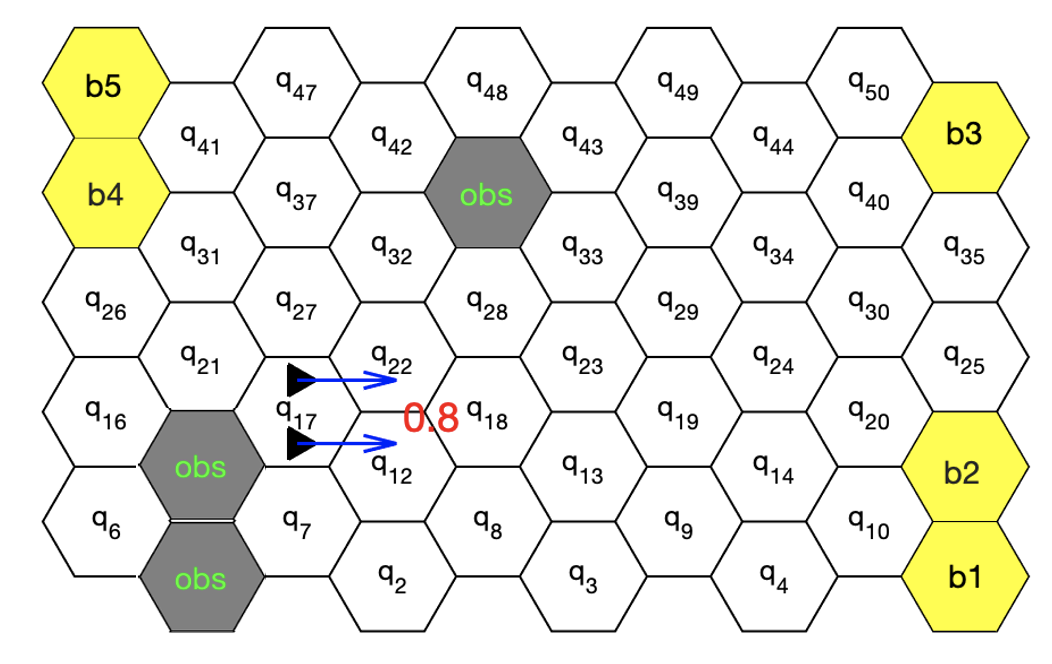}
	\caption{Hexagonal world. }
	\label{Fig:hexagonal}
\end{figure}
\vspace{-0.2cm}

\begin{example}[Hexagonal world]\label{example2}
We consider a hexagonal grid map, as shown in Fig.~\ref{Fig:hexagonal}. Hexagonal grid map representations offer several advantages over quadrangular grid maps, including lower quantisation error \cite{jeevan2018image} and enhanced performance in cooperative robot exploration \cite{quijano2007improving}.
The yellow regions, labelled as $\emph{b1}, \emph{b2}, \cdots, \emph{b5}$, are
base stations, and the grey regions, labelled as $\emph{obs}$, are obstacle regions. The state of the robot is defined as $(q_i, w)$, where $q_i$ represents the specific region where the robot is located and $w\in \{N, S, E, W\}$ represents the orientation of the robot. There are 4 action primitives $\{\textsc{FR}, \textsc{BK}, \textsc{TR}, \textsc{TL}\}$, which stand for move forward, move backward, turn right, and turn left, respectively. The robot’s motion is subject to uncertainty due to actuation noise and drifting. It is known that the probabilities of the 4 actions being executed correctly are $0.8, 0.7, 0.9,$ and $0.9$, respectively. 
Moreover, it should be noted that, depending on the precise location (which is not available due to imprecise sensing/perception) within each region and the orientation of the robot, there may be several potential target states when the action is correctly executed. For instance, as depicted in Fig.~\ref{Fig:hexagonal}, when the
robot is at state $(q_{17}, E)$ and wants to take an action $\textsc{FR}$, with probability 0.8 it ends up in state $(q_{12}, E)$ or $(q_{22}, E)$. 

In this example, it is convenient to abstract the robot dynamics
as an MDPST since one can combine the moves of all potential target regions as a set-valued transition for each correctly executed motion. For instance, when the robot's state is $(q_{17}, E)$ and it takes an action $\textsc{FR}$, then there are three possible transitions: i) with probability 0.8 it moves to the set $\{(q_{12}, E), (q_{22}, E)\}$, ii) with probability 0.1 it moves to the singleton set $\{(q_{27}, E)\}$, and iii) with probability 0.1 it moves to the singleton set $\{(q_7, E)\}$. 
\end{example}

Previously, MDPSTs were used to 
formulate the trembling-hand problem in nondeterministic domains~\cite{yu2024trembling}, where \LTLf was utilised as the specification language.
Note that, for \LTLf objectives, one can translate the formula into a deterministic finite automaton.
For \LTL goals, however, this is more involved because DFAs are not sufficient, and one has to resort to automata
over infinite traces.
This complicates the strategy synthesis procedure and requires us to develop new solution techniques 
based on the 
Winning Region (see Def. \ref{def:WR})
to capture acceptance conditions. 

Let us first recap the definitions of a \emph{feasible distribution} and \emph{nature} for MDPSTs originally introduced in \cite{yu2024trembling}.

Given a MDPST $\M$, define the set of reachable states of $(s, a)$ as $\post_{\M}(s, a) = \{s' \mid \exists \Theta\in \mdpsttr(s, a) \text{ s.t. } \probtr(s, a, \Theta)>0 \wedge  s'\in \Theta \}.$
A \emph{feasible distribution} of $\M$ guarantees that, given a state-action pair $(s,a)$, $\emph{(i)}$ the sum of probabilities of selecting a state from $\post_{\M}(s,a)$ equals 1; $\emph{(ii)}$ the sum of probabilities of selecting a state from a set $\Theta \in \F(s, a)$ equals $\T(s, a, \Theta)$;
and $\emph{(iii)}$  the probability of selecting a state outside $\post_{\M}(s,a)$ is 0. 
In the following definition, 
$\iota^{\Theta}_{s'}$ indicates whether $s'$ is \emph{in} ${\Theta}$. Hence $\iota^{\Theta}_{s'}=1$ if $s'\in \Theta$ and $\iota^{\Theta}_{s'}=0$ otherwise. Furthermore, $\alpha_{s'}^{\Theta}$ represents the probability of $s'$ being \emph{selected} from ${\Theta}$ if $s'\in \Theta$. Note that Definition 5 in \cite{yu2024trembling} only allows deterministic choice within $\Theta$ (i.e., $\alpha_{s'}^{\Theta} = 1$ if $s'$ is selected, and 0 otherwise). In this work, we also permit probabilistic choice within $\Theta$.

\begin{definition}[Feasible distribution in MDPSTs]\label{Def:feasibledistribution}
    Let $\mathcal{M}=(S, s_0, A, \mdpsttr, \probtr, \L)$ be an MDPST and $(s, a)$ a state-action pair, where $a \in A(s)$. $\mathfrak{h}_s^a\in \distr(S)$ is a \emph{feasible distribution} of $(s, a)$, denoted by $s \xrightarrow[]{a} \mathfrak{h}_s^a$, if 
    \begin{itemize}
        \item[(i)] 
        $\sum_{s'\in \Theta} \alpha_{s'}^{\Theta} =1, \mbox{ for } \Theta\in \mdpsttr(s, a)$;
        \item[(ii)] $\mathfrak{h}_x^a(s'){=}\sum\limits_{\Theta\in \mdpsttr(s, a)}\iota^{\Theta}_{s'}\alpha_{s'}^{\Theta} \probtr(s, a, \Theta), 
        \mbox{ for } s'{\in} \post_{\M}(s, a)$; 
        \item[(iii)] $\mathfrak{h}_s^a(s')=0, \mbox{ for } s'\in S\setminus \post_{\M}(s, a)$.
    \end{itemize}  
\end{definition}

A \emph{nature} is defined to characterize the unquantifiable uncertainty in MDPSTs, motivated by the definition of \emph{nature} in robust MDPs~\cite{nilim2005robust}.
One can think of nature as the strategy controlled by the adversarial environment.

\begin{definition}[Nature for MDPSTs \cite{yu2024trembling}] 
A \emph{nature} of an MDPST is a function
$\gamma: \fpaths \times A \to \distr(S)$ such that  $\gamma(\xi, a)\in \mathcal{H}_s^a$ for $\xi\in \fpaths$ and $a\in A({\ls}(\xi))$, where $\H_s^a$ is the set of feasible distributions of $(s,a)$.
\end{definition}

Suppose we fix a nature $\gamma$. The probability of an agent strategy $\sigma$ satisfying an \LTL specification $\varphi$ is denoted by 
\begin{equation*}
    \prob_{\M}^{\sigma, \gamma}(\varphi) := \prob_{\M}(\{\xi\in \Omega_{\sigma, \gamma}^\M(s_0) \mid \L(\xi) \models \varphi\}),
\end{equation*}
where $\Omega_{\sigma, \gamma}^\M(s_0)$ is the set of all probable paths generated by the agent strategy $\sigma$ and nature $\gamma$ from initial state $s_0$. 

Similarly to Definitions 7-8 in \cite{yu2024trembling}, we now define (optimal) robust strategies for MDPSTs with \LTL specifications, rather than \LTLf, which account for all possible natures.  
Of course, given a fixed strategy $\sigma$ and a nature $\gamma$, one can deduce a Markov chain $\M_{\sigma, \gamma} $ from $\M$.
 
\begin{definition}[Robust strategy]\label{RobustSat}
    Let $\M$ be an MDPST, $\varphi$ an \LTL formula, and $\beta \in [0,1]$ a threshold. An agent strategy $\sigma$ \emph{robustly enforces} $\varphi$ in ${\mathcal{M}}$ wrt $\beta$ if, for every nature $\gamma$, the probability of generating paths satisfying $\varphi$ in $\M$ is no less than $\beta$, that is, $P_{{\mathcal{M}}}^{\sigma} (\varphi)\ge \beta$, where $P_{{\mathcal{M}}}^{\sigma} (\varphi):={\min}_{\gamma}\{{\rm Pr}_{{\mathcal{M}}}^{\sigma, \gamma}(\varphi)\}.$
    Such a strategy $\sigma$ is referred to as a robust strategy for $\M$~(with respect to $\beta$).
\end{definition}

\begin{definition}[Optimal robust strategy]
    An optimal strategy $\sigma^*$ that robustly enforces an \LTL formula $\varphi$ in an MDPST $\M$ is given by  $\sigma^*=\arg{\max}_{\sigma}\{{\rm Pr}_{{\mathcal{M}}}^{\sigma} (\varphi)\}.$
     In this case, $\sigma^*$ is referred to as an \emph{optimal robust strategy} for $\M$.
\end{definition}

The optimal robust strategy synthesis problem considered in this paper is formulated as follows.

\begin{problem}[Optimal Robust Strategy Synthesis]\label{Prom1}
Given an MDPST ${\mathcal{M}}$ 
and an \LTL formula $\varphi$, synthesise an optimal robust strategy $\sigma^*$.
\end{problem}

\section{Solution technique}

We now introduce our solution to Problem~\ref{Prom1} for a given MDPST $\M$ and \LTL formula $\varphi$. Our approach is based on a reduction to the reachability problem, but is more challenging than for MDPs because of unquantifiable uncertainty, and than for \LTLf because of the need to consider infinite paths. It contains 3 steps.
\begin{itemize}
    \item[1)] Construct the product $\M^{\times}$ of the MDPST $\M$ and the LDBA $\aut$ derived from the \LTL specification $\varphi$.
    \item[2)] Reduce Problem \ref{Prom1} to a reachability problem over the product MDPST $\M^{\times}$. This step presents a significant challenge; we address it by introducing the concept of a Winning Region for MDPSTs, along with a novel algorithm for computing it.
    \item[3)] Synthesise the strategy of the reachability problem over the product $\M^{\times}$. 
\end{itemize}
We will introduce each step in detail below.

\subsection{Product MDPST}

As mentioned before, we first obtain an LDBA $\A= (Q, q_{0}, \Sigma, \trans = \trans_i\cup\trans_j\cup \trans_{acc}, \acc)$ from the LTL formula $\varphi$ using  state-of-the-art tools such as Rabinizer 4~\cite{kvretinsky2018rabinizer}.
We prefer LDBAs over DRAs because nondeterministic LDBAs are usually smaller than DRAs~\cite{sickert2016limit}.
This thus yields smaller product MDPSTs and smaller strategies.
Then, we construct the product MDPST $\M^\times$ of the MDPST $\M=(S, s_0, A,  \mdpsttr, \probtr, \L)$ and the LDBA $\A$ as follows.

\begin{definition}[Product MDPST]\label{productautomaton2}
A product MDPST is a tuple $\mathcal{M}^{\times} =(S^\times, s^\times_{0}, A^{\times}, \mdpsttr^{\times}, \probtr^{\times}, \L^{\times}, \acc^\times)$, where $S^\times=S\times Q$ is the set of states, $s^\times_{0}=(s_0, q_{0})$ is the initial state, and
\begin{itemize}
\item $A^{\times}=A\cup A^\epsilon$ where $ A^\epsilon:=\{\epsilon_q\mid q\in Q\}$;
\item $\mdpsttr^{\times}: S^\times \times A^{\times} \Mapsto 2^{2^{S^\times}}$ is the set-valued nondeterministic transition function. For every $a \in A^{\times} $, $ (s, q) \in S^{\times}$, we define $\Theta^{\times} \in \mdpsttr^{\times}((s, q), a)$ as follows:
\begin{itemize}
    \item[i)] if $a \in A$, let $q'= \delta(q, \L(s))$, and define $\Theta^{\times} = \{(s', q') \mid s'\in \Theta\}$, for every $\Theta \in \mdpsttr(s, a)$;
    \item[ii)] otherwise $ a = \epsilon_{q'}\in A^\epsilon$, then for every $q' \in \trans_j(q, \epsilon)$, define $\Theta^{\times} = \{ (s, q')\}$.
\end{itemize}
    \item $\probtr^{\times}: S^\times \times A^{\times} \times 2^{S^\times} \mapsto [0, 1]$, where
\begin{itemize}
    \item $\probtr^{\times}((s, q), a, \Theta^{\times}) = \probtr(s, a, \Theta)$, if $a\in A(s), \Theta^{\times} = \Theta \times \{q'\} $ for $\Theta \in \mdpsttr(s, a)$ where $q' = \delta(q, \L(s))$;
    \item $\probtr^{\times}((s, q), a, \Theta^{\times}) = 1$ if $a\in A^\epsilon, \Theta^{\times} = \{(s, q')\}$ for some $q' \in \trans_j(q, \epsilon)$ with $a = \epsilon_{q'}$,
    \item $\probtr^{\times}((s, q), a, \Theta^{\times}) = 0$, otherwise.
\end{itemize}
  \item $\L^{\times}: S^\times \to 2^{Prop}$, where $\L^{\times}((s, q))=\L(s)$;
  \item $\acc^\times=\{(s, q)\in S^{\times} \mid q\in \acc\}$.
\end{itemize}

An infinite path $\xi^\times$ of $\M^\times$ satisfies the B\"{u}chi condition if $\infset(\xi^\times)\cap {\rm Acc}^\times \neq \emptyset$. Such a path is said to be accepting.
\end{definition}

When an MDPST action $a\in A$ is taken in the product MDPST $\M^\times$, the alphabet used to transition the LDBA is deduced
by applying the proposition labelling function to the current
MDP state: $\L(s)\in 2^{\prop}$. In this case, the LDBA transition
$\delta(q, \L(s))$ is deterministic. Otherwise, if an $\epsilon$-transition $\epsilon_{\hat q}\in \{\epsilon_q\mid q\in Q\}$ is taken, the LDBA selects an $\epsilon$-transition, and the nondeterminism of $\trans_j(q, \epsilon)$ is resolved by transitioning the automaton state to $\hat q$. 


\subsection{Reduction to a reachability problem}

Conventional planning approaches for MDPs against \LTL specification require  computing \emph{maximal end components} (MECs) in the product and determining which MECs are accepting.
These MECs are then regarded as the goal states in the reachability planning problems. 
For MDPSTs, however, this approach is not  applicable. To better understand this, let's first review the definition of (maximal) end-components ((M)EC) for MDPs.

\begin{definition}[EC for MDPs \cite{de1998formal,baier2008principles}]
An \emph{end component (EC)} of an MDP $\M$ is a sub-MDP $\M'$ of $\M$ such that its underlying graph is strongly connected. A \emph{maximal EC (MEC)} is maximal under set inclusion.
\end{definition}

\begin{lemma}[EC properties for MDPs. Theorems 3.1 and 4.2 of \cite{de1998formal}]\label{lem:mdp-mec}
    Once an end component $E$ of an MDP $\M$ is entered, there is a strategy that i) visits every state-action pair 
    in $E$ infinitely often with probability 1, and ii) stays in $E$ forever.
\end{lemma}

Lemma~\ref{lem:mdp-mec} makes it possible to reduce a planning problem over MDPs with \LTL objectives to a reachability problem over the product MDP. This is due to the fact that, whenever a state 
$s^\times$ of an accepting MEC 
$E$ of the product MDP $\M^\times$ is reached, there exists a strategy of $\M^\times$ starting from $s^\times$ that ensures every state in 
$E$ (including the accepting states) will be visited infinitely often (according to Lemma \ref{lem:mdp-mec}), thereby satisfying the LTL objective.

For EC decomposition, MDPs can be seen as directed graphs where each state corresponds to a node and each action-labelled (probabilistic) transition corresponds to an edge.
However, this approach cannot be directly applied
to MDPSTs where transitions lead to set-valued successors rather than individual states.
For instance, consider a set-valued transition $\Theta^\times = \{s', s'', s'''\}\in \F^\times(s^\times, a)$ in a product MDPST $\M^\times$.
Here it is not sufficient to add edges for all pairs $(s^\times, s'), (s^\times, s''), (s^\times, s''')$, as the adversarial nature may prevent reaching some states in $\Theta^\times$ from $s^\times$. This poses a significant challenge in identifying the set of states in the product MDPST $\M^\times$ that are guaranteed to visit the set of accepting states $\acc^\times$ infinitely often with probability 1, an essential step in reducing the LTL planning problem to a reachability problem.
To address this challenge, we propose a procedure for identifying a set of states, called the Winning Region, in an MDPST that are guaranteed to visit a set of accepting states infinitely often with probability 1, defined formally below.

\begin{definition}[Winning Region for MDPSTs]\label{def:WR}
Given an MDPST $\M = (S, s_0, A, \mdpsttr, \probtr, \L)$ with a set of accepting states $\acc \subseteq S$, we say a set of states $W \subseteq S$ is a Winning Region (WR) for the MDPST $\M$ if, for every state $s\in W$, there exists a strategy $\sigma(s)$ starting from $s$ such that $\prob_{{\mathcal{M}}}^{\sigma(s)}(\square\diamondsuit \acc) = 1$.
\end{definition}

Next, we propose an algorithm for computing the WR $W^\times$ of the product MDPST $\M^\times$, which is outlined in Algorithm 1. The algorithm consists of the following steps.

\begin{algorithm}[t]
\caption{\textit{Compute winning region for MDPST}}
\begin{algorithmic}[1]
\Require Product MDPST $\M^\times$.
\Ensure Winning region $W^\times$.
\State Compute $S_p$ and construct sub-MDPST $\M_{sub}^\times = (S_p, s^\times_0, A^\times, \F_p, \mathcal{T}_p, \L_p, \acc^\times)$;
\State flag $= 1$;
\While {flag $= 1$ and $\acc^\times \neq \emptyset$}
\State Split ${\acc}^\times$ into two virtual copies $I_{in} = \{s^{in}: \text{$s^{in}$ is a virtual copy of $s, \forall s\in \acc^\times$}\}$ and $I_{ out}\{s^{in}: \text{$s^{out}$ is a virtual copy of $s, \forall s\in \acc^\times$}\}$;
\State $\hat{S} =(S_p \setminus {\acc}^\times)\cup I_{\rm in}\cup I_{\rm out}$;
\State Construct the MDPST $\hat{\M}_{sub}^\times$ (cf. (\ref{def:hat_mdpst})) over $\hat{S}$;
\State Compute the optimal value function $V_{sat}$ for $\hat{\M}_{sub}^\times$ with the robust dynamic programming operator $T$ in (\ref{VF_deterministic});
\If {$\exists s^{out}\in I_{out}$ s.t. $V_{sat}(s^{out})\neq 1$ }
\State flag $\leftarrow 1$;
\State Update $S_p$ and $\acc^\times$;
\Else
\State flag $\leftarrow 0$;
\EndIf
\EndWhile
\State $W^{\times} = S_p$.
\end{algorithmic}
\end{algorithm}


First, we introduce an optimisation that computes a sub-MDPST $\M^\times_{sub}$ of $\M^\times$, which includes only states that are~(forward) reachable from the initial state $s^\times_0$ and~(backward) reachable from the set of accepting states ${\acc}^\times$ (line 1). Denote by i) $S_p\subseteq S^\times$ the set of states that can be reached from both the initial and accepting states and ii) $\M^\times_{sub} = (S_p, s^\times_0, A^\times, \F_p, \mathcal{T}_p, \L^\times)$ the sub-MDPST constructed from $\M^\times$ with respect to $S_p$ (an algorithm for computing $S_p$ and $\M^\times_{sub}$ can be found in \cite{yu2024trembling}). 

Second, we iteratively remove states in $S_p$ that cannot visit $\acc^\times$ infinitely often with probability 1 (lines 2-20). Before starting the iteration, a flag is set to 1 (line 2), indicating that the iteration should proceed. Each iteration begins by splitting the set of accepting states ${\acc}^\times$ into two virtual copies: i) $I_{in}$, which only has incoming transitions into
${\acc}^\times$, and ii) $I_{out}$, which only has outgoing transitions from ${\acc}^\times$ (line 4). 
Then a new state space can be defined as $\hat{S}:=(S_p \setminus {\acc}^\times)\cup I_{in}\cup I_{out}$ (line 5).

Over $\hat S$, we can construct a new product MDPST 
\begin{equation}\label{def:hat_mdpst}
    \hat{\mathcal{M}}^\times_{sub}=(\hat{S}, s^\times_{0}, \hat{A}^{\times}, \hat{\mathcal{F}}^{\times},  \hat{\mathcal{T}}^{\times}, \hat{\mathcal{L}}^{\times})
\end{equation}
which is equivalent to ${\M}^\times_{sub}$ (line 6). For each copy $s^{\rm in}\in I_{in}$ of an accepting state $s\in \acc^\times$, we assign only a self-loop transition. As a result, each time $s^{\rm in}$ is visited, it will be visited infinitely often. 
The detailed construction of $\hat{\mathcal{M}}^\times_{sub}$ can be found in Appendix A.

We now give a robust value iteration algorithm \cite{nilim2005robust} for computing the robust maximal probability of reaching $I_{in}$ from each state $s^\times \in \hat{S}$ (line 7). Define a value function $V_{sat}: \hat{S} \to \mathbb{R}_{\ge 0}$, where
\begin{equation*}\label{valuefunction_MDPST}
\begin{aligned}
    V_{sat}(s^\times)=&\max_{\sigma^\times(s^\times)}\min_{\gamma^\times}\big\{  \\
    &\hspace{-8mm}{\rm Pr}_{\hat{\mathcal{M}}^\times_{sub}}(\{\xi^\times\in \Omega_{\sigma^\times(s^\times), \gamma^\times}^{\hat{\mathcal{M}}^\times_{sub}}(s^\times) \mid \hat{\L}^{\times}(\xi^\times) \models \diamondsuit I_{in}\})\big\},
\end{aligned}    
\end{equation*}
which represents the robust maximal probability of reaching $I_{in}$ from $s^\times$. Then one can get that $V_{sat}(s^\times)=1, \forall s^\times \in I_{in}$.

It was shown in \cite{trevizan2007planning,yu2024trembling} that a simplified Bellman equation exists for MDPSTs. Therefore, for $s^\times\in \hat{S}\setminus I_{in}$,  
the robust dynamic programming operator $T$ can be designed as
\begin{equation}\label{VF_deterministic}
\begin{aligned}
    T(V_{sat})(s^\times)
    =\max_{a\in \hat{A}^\times(s^\times)}\Big\{&\sum_{\Theta\in \hat{\mdpsttr}^{\times}(s^\times, a)}\hat{\probtr}^\times(s^\times, a, \Theta) \\
    & \min_{s'\in \Theta}\{V_{sat}(s')\}\Big\}.
    \end{aligned}
\end{equation}

Once the robust value iteration converges and thus the optimal value function $V_{sat}$ is obtained, we first check whether there exists a state $s^{out} \in I_{out}$ such that $V_{sat}(s^{out}) \neq 1$. If such a state exists, we set flag to 1, and then remove the corresponding state $s$ from both $S_p$ and $\acc^\times$ (lines 8-10). Otherwise, the flag is set to 0 (11-12). The iteration continues if the flag is 1 and terminates once flag becomes 0 or $\acc^\times = \emptyset$. Once the iteration terminates, the algorithm returns the WR $W^\times$ (line 15). Our main result then follows.

\begin{theorem}\label{thm1}
    Given an MDPST ${\mathcal{M}}$ and an \LTL formula $\varphi$, the maximal probability of satisfying $\varphi$ is given by
    \begin{equation}\label{thm:correctness}
    \begin{aligned}
        \max_{\sigma\in \Pi_{\M}}\{\prob_{{\M}}^\sigma(\varphi)\}
        = \max_{{{\sigma^\times}}\in {\Pi_{{\M}^\times}}} \{\prob_{{{\M}}^\times}^{{{\sigma^\times}}}( \diamondsuit W^\times)\},
    \end{aligned}       
    \end{equation}
    where $W^\times$ is the WR computed by Algorithm 1.
\end{theorem}

\begin{proof}[Proof Sketch]
    We prove Theorem \ref{thm1} in two steps. First, we show the correctness of Algorithm 1, i.e., that the output $W^\times$ of Algorithm 1 is indeed the WR of the product MDPST $\M^\times$. Then, we show that (\ref{thm:correctness}) holds by verifying both sides of the inequality and subsequently constructing the induced policy on $\M$. 
\end{proof}

An example illustrating the steps of Algorithm 1 are provided in Appendix B. The full proof of Theorem \ref{thm1} can be found in Appendix C.



\subsection{Optimal robust strategy synthesis}\label{Sec:ors}

We have thus reduced Problem \ref{Prom1} to the reachability problem over $\M^\times_{sub}$, where the goal set is given by the WR $W^\times$. 
For states $s^\times \in W^\times,$ it holds that $V_{sat}(s^\times)=1$. For states $s^\times \in S_p\setminus W^\times$, the optimal value function $V_{sat}$ can be determined by conducting another run of the robust value iteration algorithm (\ref{VF_deterministic}).
The optimal robust strategy $\sigma^\times$ can be derived from $V_{sat}$ using standard methods.

\section{Experiments}

In this section, a case study is provided to demonstrate the effectiveness of our method. 
We implemented the solution technique proposed in Section IV in Python, and use Rabinizer 4 \cite{kvretinsky2018rabinizer} for the LTL-to-LDBA construction. For the robust value iteration, we set the convergence threshold to $10^{-3}$, i.e., the value iteration stops when $\max_{s\in S_l}\{|V^{k+1}_{sat}(s)-V^{k}_{sat}(s)|\}< 10^{-3}$. 
All simulations are carried out on a Macbook Pro (2.6 GHz 6-Core Intel Core i7 and 16 GB of RAM) and the implementation
code can be found at: \href{https://github.com/piany/MDPST-full-LTL}{https://github.com/piany/MDPST-full-LTL}.

We consider a mobile robot moving in the hexagonal world described in Example \ref{example2}, where the size of the workspace is denoted by $(N_x, N_y)$.
As explained, the robot dynamics can be abstracted as an MDPST.  The robot is required to persistently survey three goal regions while avoiding obstacles at all times. This task is expressed as the \LTL formula 
\begin{equation*}
\begin{aligned}
   & \varphi_{persistavoid}=\\
    & \hspace{3mm}(\square\lozenge \texttt{b1} \vee \texttt{b2}) \wedge (\square\lozenge \texttt{b3}) \wedge (\square\lozenge \texttt{b4} \vee \texttt{b5}) \wedge (\square\neg \texttt{obs}).
    \end{aligned}
\end{equation*}

The corresponding LDBA derived using Rabinizer 4 has 4 states. For the scenario 
$(N_x, N_y)= (10, 5)$, the constructed product MDPST  
$\mathcal{M}^\times$ has 800 states and 5440 (single and set-valued) edges. The WR $W^\times$ is computed using Algorithm 1, which has 509 states. The initial state of the robot is $(q_1, N)$ and one can compute that $\max\{{\rm Pr}_{\mathcal{M}^\times} (\varphi_{persistavoid})\}=0.85$. The (MDPST, LDBA, and product MDPST) model construction took in total 0.467s and the strategy synthesis took 9.144s.

To highlight the computational advantage of LDBA over DRA, we also consider DRA as representations for the LTL task $\varphi_{persistavoid}$. The resulting DRA has 8 states, whereas the LDBA has only 4 states. 
We compare the performance of both representations in three scenarios: $(N_x, N_y) = (10, 5), (N_x, N_y) = (16, 8)$, and $(N_x, N_y) = (20, 10)$.
TABLE I shows the number of states and transitions  $(|S^\times|, |\mathcal{T}^\times|)$ of the product MDPST $\mathcal{M}^\times$, along with the model construction time $T_{mdl}$, and the strategy synthesis time $T_{sys}$ for each scenario $(N_x, N_y)$. The results clearly show that strategy synthesis with LDBA is significantly faster than with DRA, particularly as the workspace size increases.

\begin{table}[]
\caption{The number of states and transitions  $(|S^\times|, |\mathcal{T}^\times|)$ of the product MDPST $\mathcal{M}^\times$, the model construction time $T_{mdl}$, and the strategy synthesis time $T_{sys}$ for different automation choices $\mathcal{A}$ and different scenarios $(N_x, N_y)$.}
\begin{tabular}{c|c|c|c|c}
\hline
$(N_x, N_y)$              &  $\mathcal{A}$    & $(|S^\times|, |\mathcal{T}^\times|)$  & $T_{mdl} (s)$ & $T_{sys} (s)$ \\ \hline
\multirow{2}{*}{(10, 5)}  & LDBA & (800, 5440)                                                         & 0.888     & 10.745      \\ \cline{2-5} 
                          & DRA  & (1600, 10880)                                                       & 0.986      & 17.363     \\ \hline
\multirow{2}{*}{(16, 8)}  & LDBA & (2048, 14344)                                                       & 0.975       & 88.529     \\ \cline{2-5} 
                          & DRA  & (4098, 28688)                                                       & 1.696       & 231.343    \\ \hline
\multirow{2}{*}{(20, 10)} & LDBA & (3200, 22728)                                                       & 2.058       & 288.662    \\ \cline{2-5} 
                          & DRA  & (6400, 45456)                                                       & 1.962       & 748.126   \\ \hline
\end{tabular}
\end{table}



\begin{figure}
\centering
\includegraphics[width=0.42\textwidth]{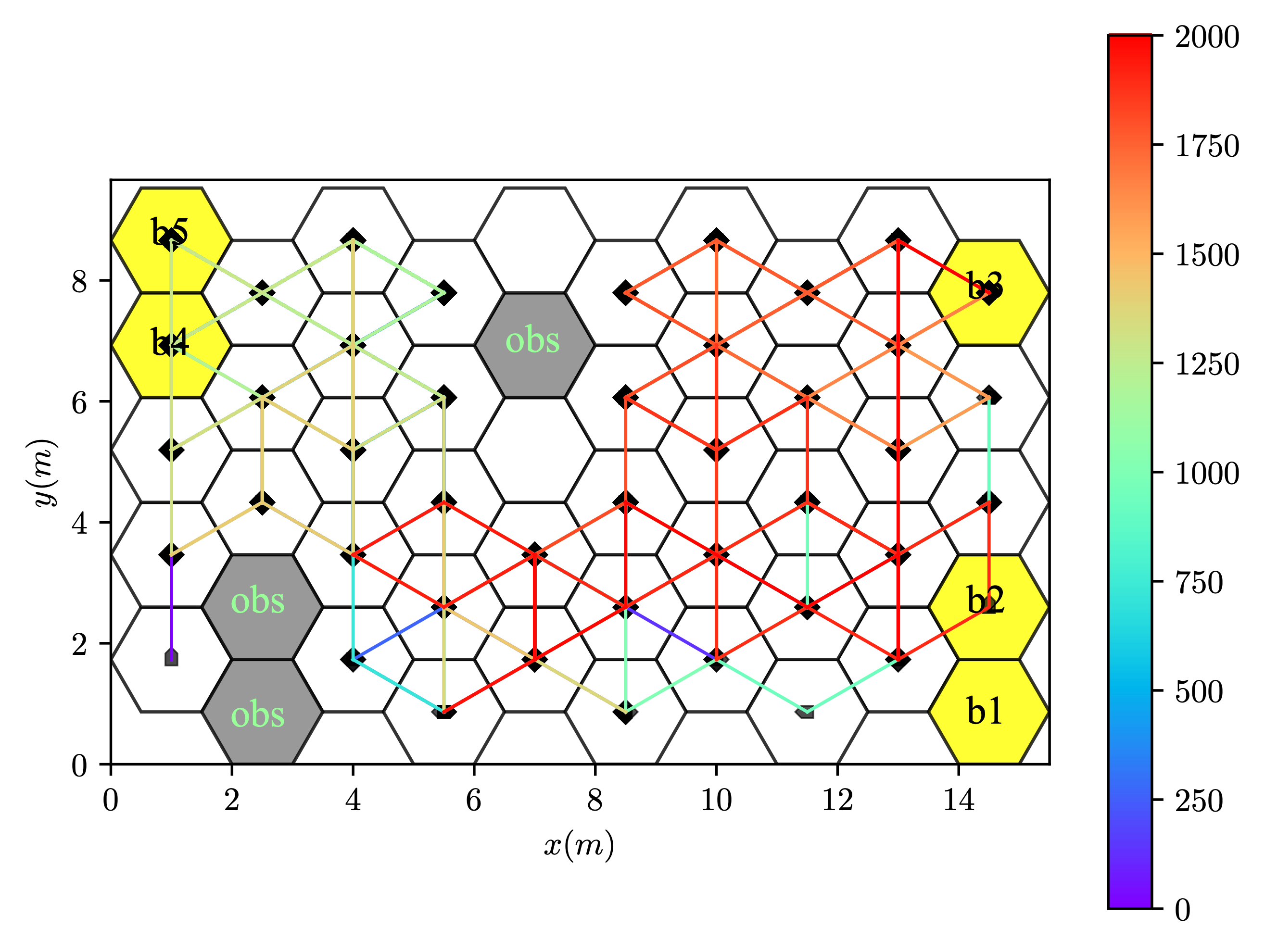}
	\caption{\footnotesize  Simulated trajectory of 2000 time steps for the LTL task $\varphi_{persistavoid}$, where the color bar denotes the time steps. }
	\label{Fig:gridworld_obs}
\end{figure}

To verify robustness, we perform
1000 Monte Carlo simulations of 2000 time steps for scenario $(N_x, N_y)= (10, 5)$. For each simulation, we randomly choose the set of parameters $\{\alpha_{s'}^{\theta}: \theta\in \mdpsttr^{\times}(s, a), s'\in \theta\}$ for each state-action pair $(s, a)$, to resolve the uncertainty for the set-valued transitions. We adopt
the optimal robust strategy for the strategy prefix and the Round-Robin strategy once the system enters the WR. The task $\varphi_{persistavoid}$ is satisfied 868 times out of the 1000 simulations, which verifies the probabilistic satisfaction guarantee.  Fig. \ref{Fig:gridworld_obs} depicts one of the simulated trajectories. One can see that the LTL specification $\varphi_{persistavoid}$ is satisfied.

\section{Conclusion}

This work studied the robot planning problem under both quantifiable and unquantifiable uncertainty, and subject to high-level \LTL task specifications. MDPSTs were proposed as a unified modelling framework for handling both types of uncertainties.  Additionally, a sound solution technique was introduced for synthesising the optimal robust strategy for MDPSTs with \LTL specifications.
For future work, we plan to explore the plan synthesis problem for MDPSTs, subject to both temporal logic specifications
and cost constraints.

\bibliographystyle{IEEEtran}
\bibliography{references,ref}

\begin{thebibliography}{10}
\providecommand{\url}[1]{#1}
\csname url@samestyle\endcsname
\providecommand{\newblock}{\relax}
\providecommand{\bibinfo}[2]{#2}
\providecommand{\BIBentrySTDinterwordspacing}{\spaceskip=0pt\relax}
\providecommand{\BIBentryALTinterwordstretchfactor}{4}
\providecommand{\BIBentryALTinterwordspacing}{\spaceskip=\fontdimen2\font plus
\BIBentryALTinterwordstretchfactor\fontdimen3\font minus \fontdimen4\font\relax}
\providecommand{\BIBforeignlanguage}[2]{{%
\expandafter\ifx\csname l@#1\endcsname\relax
\typeout{** WARNING: IEEEtran.bst: No hyphenation pattern has been}%
\typeout{** loaded for the language `#1'. Using the pattern for}%
\typeout{** the default language instead.}%
\else
\language=\csname l@#1\endcsname
\fi
#2}}
\providecommand{\BIBdecl}{\relax}
\BIBdecl

\bibitem{puterman2014markov}
M.~L. Puterman, \emph{{Markov decision processes: discrete stochastic dynamic programming}}.\hskip 1em plus 0.5em minus 0.4em\relax John Wiley \& Sons, 2014.

\bibitem{natarajan2022planning}
M.~Natarajan and A.~Kolobov, \emph{{Planning with Markov decision processes: An AI perspective}}.\hskip 1em plus 0.5em minus 0.4em\relax Springer Nature, 2022.

\bibitem{CimattiPRT03}
A.~Cimatti, M.~Pistore, M.~Roveri, and P.~Traverso, ``Weak, strong, and strong cyclic planning via symbolic model checking,'' \emph{Artif. Intell.}, vol. 147, no. 1-2, pp. 35--84, 2003.

\bibitem{FOND}
M.~Ghallab, D.~S. Nau, and P.~Traverso, \emph{Automated planning - theory and practice}, 2004.

\bibitem{GeffnerBonet2013}
H.~Geffner and B.~Bonet, \emph{A Concise Introduction to Models and Methods for Automated Planning}.\hskip 1em plus 0.5em minus 0.4em\relax Morgan \& Claypool Publishers, 2013.

\bibitem{thrun2005probabilistic}
S.~Thrun, W.~Burgard, and D.~Fox, ``{Probabilistic Robotics},'' 2005.

\bibitem{du2011robot}
N.~E. Du~Toit and J.~W. Burdick, ``{Robot motion planning in dynamic, uncertain environments},'' \emph{IEEE Transactions on Robotics}, vol.~28, no.~1, pp. 101--115, 2011.

\bibitem{ajoudani2018progress}
A.~Ajoudani, A.~M. Zanchettin, S.~Ivaldi, A.~Albu-Sch{\"a}ffer, K.~Kosuge, and O.~Khatib, ``Progress and prospects of the human--robot collaboration,'' \emph{Autonomous robots}, vol.~42, pp. 957--975, 2018.

\bibitem{trevizan2007planning}
F.~W. Trevizan, F.~G. Cozman, and L.~N. de~Barros, ``{Planning under Risk and Knightian Uncertainty},'' in \emph{IJCAI}, vol. 2007, 2007, pp. 2023--2028.

\bibitem{trevizan2008mixed}
F.~W. Trevizan, F.~G. Cozman, and L.~N. De~Barros, ``{Mixed probabilistic and nondeterministic factored planning through Markov decision processes with set-valued transitions},'' in \emph{Workshop on A Reality Check for Planning and Scheduling Under Uncertainty at ICAPS}, 2008, p.~62.

\bibitem{white1994markov}
C.~C. White~III and H.~K. Eldeib, ``{Markov decision processes with imprecise transition probabilities},'' \emph{Operations Research}, vol.~42, no.~4, pp. 739--749, 1994.

\bibitem{satia1973markovian}
J.~K. Satia and R.~E. Lave~Jr, ``{Markovian decision processes with uncertain transition probabilities},'' \emph{Operations Research}, vol.~21, no.~3, pp. 728--740, 1973.

\bibitem{givan2000bounded}
R.~Givan, S.~Leach, and T.~Dean, ``{Bounded-parameter Markov decision processes},'' \emph{Artificial Intelligence}, vol. 122, no. 1-2, pp. 71--109, 2000.

\bibitem{nilim2005robust}
A.~Nilim and L.~El~Ghaoui, ``{Robust control of Markov decision processes with uncertain transition matrices},'' \emph{Operations Research}, vol.~53, no.~5, pp. 780--798, 2005.

\bibitem{buffet2005robust}
O.~Buffet and D.~Aberdeen, ``{Robust planning with (L) RTDP},'' in \emph{Proceedings of the 19th international joint conference on Artificial intelligence}, 2005, pp. 1214--1219.

\bibitem{hahn2019interval}
E.~M. Hahn, V.~Hashemi, H.~Hermanns, M.~Lahijanian, and A.~Turrini, ``{Interval Markov decision processes with multiple objectives: from robust strategies to Pareto curves},'' \emph{ACM Transactions on Modeling and Computer Simulation (TOMACS)}, vol.~29, no.~4, pp. 1--31, 2019.

\bibitem{Con93}
A.~Condon, ``On algorithms for simple stochastic games,'' \emph{Advances in computational complexity theory, DIMACS Series in Discrete Mathematics and Theoretical Computer Science}, vol.~13, pp. 51--73, 1993.

\bibitem{yu2024trembling}
P.~Yu, S.~Zhu, G.~De~Giacomo, M.~Kwiatkowska, and M.~Vardi, ``{The trembling-hand problem for LTL$_f$ planning},'' in \emph{Proceedings of the 33rd International Joint Conference on Artificial Intelligence}, 2024, pp. 3631--3641.

\bibitem{de2013linear}
G.~De~Giacomo and M.~Y. Vardi, ``{Linear temporal logic and linear dynamic logic on finite traces},'' in \emph{Proceedings of the 23rd International Joint Conference on Artificial Intelligence}, vol.~13, 2013, pp. 854--860.

\bibitem{Pnu77}
A.~Pnueli, ``{The temporal logic of programs},'' in \emph{{FOCS}}, 1977, pp. 46--57.

\bibitem{ding2014ltl}
X.~Ding, M.~Lazar, and C.~Belta, ``{LTL receding horizon control for finite deterministic systems},'' \emph{Automatica}, vol.~50, no.~2, pp. 399--408, 2014.

\bibitem{schillinger2019hierarchical}
P.~Schillinger, M.~B{\"u}rger, and D.~V. Dimarogonas, ``{Hierarchical LTL-task mdps for multi-agent coordination through auctioning and learning},'' \emph{The international journal of robotics research}, 2019.

\bibitem{guo2018probabilistic}
M.~Guo and M.~M. Zavlanos, ``{Probabilistic motion planning under temporal tasks and soft constraints},'' \emph{IEEE Transactions on Automatic Control}, vol.~63, no.~12, pp. 4051--4066, 2018.

\bibitem{sadigh2014learning}
D.~Sadigh, E.~S. Kim, S.~Coogan, S.~S. Sastry, and S.~A. Seshia, ``{A learning based approach to control synthesis of markov decision processes for linear temporal logic specifications},'' in \emph{53rd IEEE Conference on Decision and Control}.\hskip 1em plus 0.5em minus 0.4em\relax IEEE, 2014, pp. 1091--1096.

\bibitem{wen2016probably}
M.~Wen and U.~Topcu, ``Probably approximately correct learning in stochastic games with temporal logic specifications.'' in \emph{Proceedings of the 25th International Joint Conference on Artificial Intelligence}, 2016, pp. 3630--3636.

\bibitem{hasanbeig2019reinforcement}
M.~Hasanbeig, Y.~Kantaros, A.~Abate, D.~Kroening, G.~J. Pappas, and I.~Lee, ``{Reinforcement learning for temporal logic control synthesis with probabilistic satisfaction guarantees},'' in \emph{2019 IEEE 58th conference on decision and control (CDC)}.\hskip 1em plus 0.5em minus 0.4em\relax IEEE, 2019, pp. 5338--5343.

\bibitem{kwiatkowska2013automated}
M.~Kwiatkowska and D.~Parker, ``{Automated verification and strategy synthesis for probabilistic systems},'' in \emph{Automated Technology for Verification and Analysis: 11th International Symposium, ATVA 2013, Hanoi, Vietnam, October 15-18, 2013. Proceedings}.\hskip 1em plus 0.5em minus 0.4em\relax Springer, 2013, pp. 5--22.

\bibitem{kwiatkowska2022probabilistic}
M.~Kwiatkowska, G.~Norman, and D.~Parker, ``{Probabilistic model checking and autonomy},'' \emph{Annual review of control, robotics, and autonomous systems}, vol.~5, no.~1, pp. 385--410, 2022.

\bibitem{BrafmanGP18}
R.~I. Brafman, G.~{De Giacomo}, and F.~Patrizi, ``{LTL$_f$/LDL$_f$} non-markovian rewards,'' in \emph{{AAAI}}, S.~A. McIlraith and K.~Q. Weinberger, Eds., 2018, pp. 1771--1778.

\bibitem{baier2008principles}
C.~Baier and J.-P. Katoen, \emph{{Principles of model checking}}.\hskip 1em plus 0.5em minus 0.4em\relax MIT press, 2008.

\bibitem{de1998formal}
L.~De~Alfaro, ``{Formal verification of probabilistic systems},'' Ph.D. dissertation, Stanford University, 1998.

\bibitem{courcoubetis1995complexity}
C.~Courcoubetis and M.~Yannakakis, ``{The complexity of probabilistic verification},'' \emph{Journal of the ACM (JACM)}, vol.~42, no.~4, pp. 857--907, 1995.

\bibitem{hahn2013lazy}
E.~M. Hahn, G.~Li, S.~Schewe, A.~Turrini, and L.~Zhang, ``{Lazy probabilistic model checking without determinisation},'' in \emph{26th International Conference on Concurrency Theory, CONCUR 2015}, 2015, pp. 354--367.

\bibitem{sickert2016limit}
S.~Sickert, J.~Esparza, S.~Jaax, and J.~K{\v{r}}et{\'\i}nsk{\`y}, ``{Limit-deterministic B{\"u}chi automata for linear temporal logic},'' in \emph{International Conference on Computer Aided Verification}.\hskip 1em plus 0.5em minus 0.4em\relax Springer, 2016, pp. 312--332.

\bibitem{DBLP:journals/iandc/VardiW94}
M.~Y. Vardi and P.~Wolper, ``{Reasoning about infinite computations},'' \emph{Inf. Comput.}, vol. 115, no.~1, pp. 1--37, 1994.

\bibitem{kvretinsky2018rabinizer}
J.~K{\v{r}}et{\'\i}nsk{\`y}, T.~Meggendorfer, S.~Sickert, and C.~Ziegler, ``{Rabinizer 4: from LTL to your favourite deterministic automaton},'' in \emph{International Conference on Computer Aided Verification}.\hskip 1em plus 0.5em minus 0.4em\relax Springer, 2018, pp. 567--577.

\bibitem{jeevan2018image}
K.~Jeevan and S.~Krishnakumar, ``An image steganography method using pseudo hexagonal image,'' \emph{Int. J. Pure Appl. Math}, vol. 118, no.~18, pp. 2729--2735, 2018.

\bibitem{quijano2007improving}
H.~J. Quijano and L.~Garrido, ``Improving cooperative robot exploration using an hexagonal world representation,'' in \emph{Electronics, Robotics and Automotive Mechanics Conference (CERMA 2007)}.\hskip 1em plus 0.5em minus 0.4em\relax IEEE, 2007, pp. 450--455.

\end{thebibliography}



\newpage

\section*{Appendix}

\subsection{Detailed construction of $\hat{\M}_{sub}^\times$}

Recall that we split the set of accepting states ${\acc}^\times$ into two virtual copies: i) $I_{in}$, which only has incoming transitions into
${\acc}^\times$, and ii) $I_{out}$, which only has outgoing transitions from ${\acc}^\times$. Then the new state-space can be defined as 
\begin{equation*}
    \hat{S}:=(S_p \setminus {\acc}^\times)\cup I_{in}\cup I_{out}
\end{equation*}

Over $\hat S$, one can further construct a new product MDPST 
\begin{equation}\label{def:hat_mdpst}
    \hat{\mathcal{M}}^\times_{sub}=(\hat{S}, s^\times_{0}, \hat{A}^{\times}, \hat{\mathcal{F}}^{\times},  \hat{\mathcal{T}}^{\times}, \hat{\mathcal{L}}^{\times}),
\end{equation}
where 
\begin{itemize}
    \item $\hat{A}^{\times}=A^\times\cup \{\tau_0\}$ and $\tau_0$ is a self-loop action. The set of available actions $\hat{A}^{\times}(s)$ is defined as $\hat{A}^{\times}(s)=A^\times(s), \forall s\in S_p \setminus \acc^\times$, $\hat{A}^{\times}(s^{\rm out})=A^\times(s), \forall s \in \acc^\times$, and $\hat{A}^{\times}(s^{\rm in})=\tau_0, \forall s\in \acc^\times$, where $s^{\rm out}$ and $s^{\rm in}$ are respectively the virtual copies of $s\in \acc^\times$ in $I_{in}$ and $ I_{out}$.
    \item $\hat{\F}^{\times}: \hat{S} \times \hat{A}^{\times} \Mapsto 2^{2^{\hat{S}^\times}}$ is the set-valued nondeterministic transition function. To define $\hat{\F}^{\times}$, we let $\Phi = \cup_{s\in S_p}\cup_{a\in {A}^{\times}} \cup_{\Theta\in \F_p(s, a)}\{\Theta\}$ be the set of all possible target (single- or set-valued) states originating from ${S}_p$. For each set $\Theta \in \Phi$, we define a \emph{copy} $\hat{\Theta}$ of $\Theta$ as i) $\hat{\Theta} = \Theta$ if $\Theta \subseteq S_p \setminus \acc^\times$, ii) $\hat{\Theta} = \{s^{\rm in}: s\in \Theta\}$ if $\Theta \subseteq \acc^\times$, and iii) $\hat{\Theta} = \{s: s\in \Theta \wedge s\in S_p \setminus \acc^\times\} \cup \{s^{\rm in}: s\in \Theta \wedge s \in \acc^\times\}$, otherwise. Then, one has that
    \begin{itemize}
        \item if $s\in S_p\setminus \acc^\times$, $\hat{\Theta} \in \hat{\F}^{\times}(s, a) \ \text{iff} \ \Theta \in \F_p(s, a)$;
        \item if $s^{\rm out}\in I_{out}$, $\hat{\Theta} \in \hat{\F}^{\times}(s^{\rm out}, a) \ \text{iff} \ \Theta \in \F_p(s, a)$;
        \item if $s^{\rm in}\in I_{in}$, $\hat{\F}^{\times}(s^{\rm in}, \tau_0) = s^{\rm in}$.
    \end{itemize}
   \item  $\hat{\mathcal{T}}^{\times}: \hat{S}\times \hat{A}^{\times} \times {2^{\hat{S}^\times}}\mapsto (0, 1]$ is the transition probability function, given by
   \begin{itemize}
       \item $\hat{\mathcal{T}}^{\times}(s, a, \hat{\Theta})= \mathcal{T}_p(s, a, {\Theta})$ for $s\in S_p\setminus \acc^\times$,
       \item $\hat{\mathcal{T}}^{\times}(s^{\rm out}, a, \hat{\Theta})= \mathcal{T}_p(s, a, {\Theta})$ for $s^{\rm out}\in I_{out}$,
       \item $\hat{\mathcal{T}}^{\times}(s^{\rm in}, \tau_0, s^{\rm in})= 1$ for $s^{\rm in}\in I_{in}$.       
   \end{itemize}
   \item $\hat{\L}^{\times}: \hat{S} \to 2^{Prop}$, where $\hat{\L}^{\times}(s)=\L^\times(s)$ if $s\in S_p\setminus \acc^\times$ and $\hat{\L}^{\times}(s^{\rm out})=\hat{\L}^{\times}(s^{\rm in})=\L^\times(s)$ otherwise.
\end{itemize}

Note that $\hat{\mathcal{M}}^\times_{sub}$ is equivalent to ${\M}^\times_{sub}$ in the sense that there exists a one-to-one correspondence between the transitions of $\hat{\mathcal{M}}^\times_{sub}$ and ${\M}^\times_{sub}$. 

\subsection{Example for demonstrating Algorithm 1}

\begin{example}
    Consider a product MDPST ${\M}^\times$ with the set of states $S^\times = \{S_1, \cdots, S_{10}\}$ and the set of actions $A^\times = \{a, b\}$, as shown in Fig. {\ref{fig:mdpst_v1}}. The initial state is $s_0^\times = S_1$ and the set of accepting states is given by $\acc^\times = \{S_4, S_5\}$ (marked in green). For the state $S_2$, both actions $a$ and $b$ are applicable and the associated (action labelled) probabilistic transitions are depicted. For the remaining states, only action $a$ is applicable and the associated  probabilistic transitions are depicted (where the action label is omitted). Notice that state $S_{10}$ is not backward reachable from the set of accepting states $\acc^\times$, therefore it follows that $S_p = \{S_1, \cdots, S_{9}\}$ and the resulting sub-MDPST ${\M}^\times_{sub}$ is shown within the shaded box of Fig. {\ref{fig:mdpst_v1}}.

    In the first iteration of Algorithm 1, one starts by splitting $\acc^\times$ into two virtual copies of $I_{in}= \{S_4^{in}, S_5^{in}\}$ and $I_{out}= \{S_4^{out}, S_5^{out}\}$. Then the new state space $\hat{S} = \{S_1, \cdots, S_3, S_4^{in}, S_5^{in}, S_4^{out}, S_5^{out}, S_6, \cdots, S_9\}$ and the constructed MDPST $\hat{\M}^\times_{sub}$ is drawn in Fig. \ref{fig:mdpst_v2}. Let $V_{sat}(S_4^{in})=V_{sat}(S_5^{in}) = 1$. With the robust dynamic programming operator $T$ in (\ref{VF_deterministic}), one can compute that $V_{sat}(S_1)=0.8, V_{sat}(S_2) = V_{sat}(S_3) = V_{sat}(S_4^{out}) = 1$, $V_{sat}(S_5^{out})= V_{sat}(S_6) = V_{sat}(S_7) = V_{sat}(S_8) = 0$, and $V_{sat}(S_9) = 0.94$. Therefore, we remove state $S_5$ from both $\acc^{\times}$ and $S_p$, and further remove states $S_1, S_6, S_7, S_8, S_9$ from $S_p$, and thus obtain $S_p = \{S_2, S_3, S_4\}$ and $\acc^{\times} = \{S_4\}$ for the second iteration.

    The product MDPST ${\M}^\times_{sub}$ for the second iteration is depicted in Fig. \ref{fig:WR}. By repeating the process in iteration 1, one can get that $V_{sat}(S_2) = V_{sat}(S_3) = V_{sat}(S_4^{in}) = V_{sat}(S_4^{out}) = 1$. Thus, one has that flag$=0$. The iteration terminates. The WR is then given by $W^\times = \{S_2, S_3, S_4\}$.
\end{example}

\begin{figure}[ht]
\centering
\begin{tikzpicture}

\fill[gray!20] (0, -5.7) rectangle (8.5, 2.2);
\node[circle, draw, fill=blue!20, minimum width=1cm, minimum height=0.5cm] (S1) at (0.5, 0) {$S_1$};
\node[circle, draw, fill=blue!20, minimum width=1cm, minimum height=0.5cm] (S2) at (3, 0) {$S_2$};
\node[circle, draw, fill=blue!20, minimum width=1cm, minimum height=0.5cm] (S3) at (5, 1.5) {$S_3$};
\node[circle, draw, fill=green!50, minimum width=1cm, minimum height=0.5cm] (S4) at (7.5, 0) {$S_4$};
\node[circle, draw, fill=green!50, minimum width=1cm, minimum height=0.5cm] (S5) at (7.5, -1.5) {$S_5$};
\node[circle, draw, fill=blue!20, minimum width=1cm, minimum height=0.5cm] (S6) at (5.75, -5) {$S_6$};
\node[circle, draw, fill=blue!20, minimum width=1cm, minimum height=0.5cm] (S7) at (5.75, -1.5) {$S_7$};
\node[circle, draw, fill=blue!20, minimum width=1cm, minimum height=0.5cm] (S8) at (3, -2.5) {$S_9$};
\node[circle, draw, fill=blue!20, minimum width=1cm, minimum height=0.5cm] (S9) at (3, -4) {$S_8$};
\node[circle, draw, fill=blue!20, minimum width=1cm, minimum height=0.5cm] (S10) at (1.5, 3) {$S_{10}$};

\draw[->, thick] (S1) -- node[midway, left, blue] {$0.2$} (S10); 
\draw[->, thick] (S1) -- node[midway, above, blue] {$0.8$} (S2); 
\draw[->, thick] (S2) -- node[midway, above, blue] {$b, 0.2$} (S4); 
\draw[->, thick] (S2) -- node[midway, above left, blue] {$b, 0.8$} (S3);
\draw[->, thick] (S2) -- node[midway, below, blue] {$a, 1$} (6.5, -1);
\draw[->, thick] (S3) -- node[midway, below, blue] {1} (S2); 
\draw[->, thick] (S4) -- node[midway, above, blue] {1} (S3); 
\draw[->, thick] (S5) -- node[midway, above, blue] {1} (S6); 
\draw[->, thick] (S6) -- node[midway, below, blue] {1} (4, -3.25);
\draw[->, thick] (S7) -- node[midway, below left, blue] {1} (S6);
\draw[->, thick] (S9) -- node[midway, above, blue] {0.6} (6.625, -2.5);
\draw[->, thick] (S9) -- node[midway, below, blue] {0.4} (S6);
\draw[->, thick] (S8) -- node[midway, above, blue] {0.3} (S1);
\draw[->, thick] (S8) -- node[midway, right, blue] {0.7} (S2);

\draw[->, thick] (.5, -1) -- (S1);

\draw[red, dashed, thick] (7.5, -0.75) ellipse (1cm and 1.8cm);
\draw[red, dashed, thick] (6.625, -1.5) ellipse (1.8cm and 1cm);
\draw[red, dashed, thick] (3, -3.25) ellipse (1cm and 1.8cm);

\draw[->, thick] (S10) edge [loop right] node[right, blue] {1} (S10);
\end{tikzpicture}
\caption{The product MDPST $\M^\times$, where the sub-MDPST $\M_{sub}^\times$ (for the first iteration) is the part within the shaded box. }
\label{fig:mdpst_v1}
\end{figure}
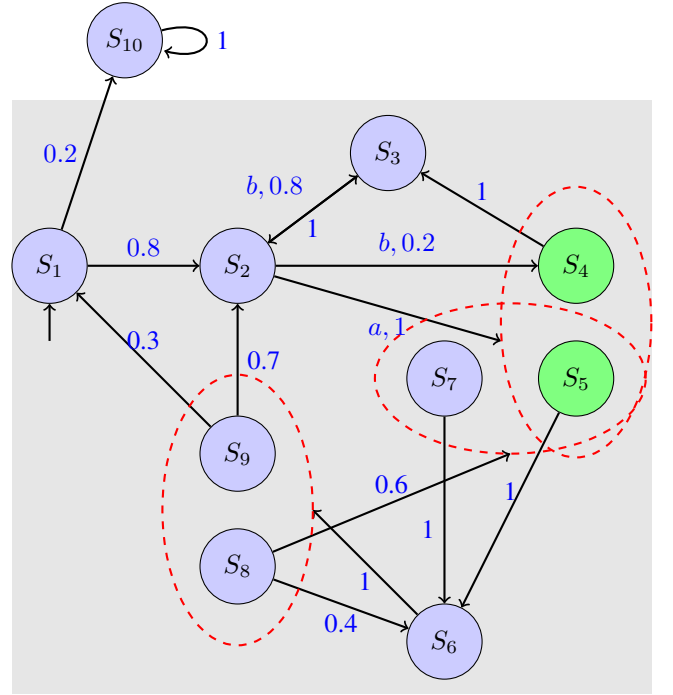

\begin{figure}[t]
\centering
\begin{tikzpicture}

\node[circle, draw, fill=blue!20, minimum width=1cm, minimum height=0.5cm] (S1) at (0.5, 0) {$S_1$};
\node[circle, draw, fill=blue!20, minimum width=1cm, minimum height=0.5cm] (S2) at (3, 0) {$S_2$};
\node[circle, draw, fill=blue!20, minimum width=1cm, minimum height=0.5cm] (S3) at (5, 1.5) {$S_3$};
\node[circle, draw, fill=green!50, minimum width=1cm, minimum height=0.5cm] (S4) at (7.5, 0) {$S_4^{in}$};
\node[circle, draw, fill=blue!20, minimum width=1cm, minimum height=0.5cm] (S42) at (7.5, 2) {$S_4^{out}$};
\node[circle, draw, fill=green!50, minimum width=1cm, minimum height=0.5cm] (S5) at (7.5, -1.5) {$S_5^{in}$};
\node[circle, draw, fill=blue!20, minimum width=1cm, minimum height=0.5cm] (S52) at (7.5, -4.5) {$S_5^{out}$};
\node[circle, draw, fill=blue!20, minimum width=1cm, minimum height=0.5cm] (S6) at (5.75, -5) {$S_6$};
\node[circle, draw, fill=blue!20, minimum width=1cm, minimum height=0.5cm] (S7) at (5.75, -1.5) {$S_7$};
\node[circle, draw, fill=blue!20, minimum width=1cm, minimum height=0.5cm] (S8) at (3, -2.5) {$S_9$};
\node[circle, draw, fill=blue!20, minimum width=1cm, minimum height=0.5cm] (S9) at (3, -4) {$S_8$};

\draw[->, thick] (S1) -- node[midway, above, blue] {$0.8$} (S2); 
\draw[->, thick] (S2) -- node[midway, above, blue] {$b, 0.2$} (S4); 
\draw[->, thick] (S2) -- node[midway, above left, blue] {$b, 0.8$} (S3);
\draw[->, thick] (S2) -- node[midway, below, blue] {$a, 1$} (6.5, -1);
\draw[->, thick] (S3) -- node[midway, below, blue] {1} (S2); 
\draw[->, thick] (S42) -- node[midway, above, blue] {1} (S3); 
\draw[->, thick] (S52) -- node[midway, below right, blue] {1} (S6); 
\draw[->, thick] (S6) -- node[midway, below, blue] {1} (4, -3.25);
\draw[->, thick] (S7) -- node[midway, below left, blue] {1} (S6);
\draw[->, thick] (S9) -- node[midway, above, blue] {0.6} (6.625, -2.5);
\draw[->, thick] (S9) -- node[midway, below, blue] {0.4} (S6);
\draw[->, thick] (S8) -- node[midway, above, blue] {0.3} (S1);
\draw[->, thick] (S8) -- node[midway, right, blue] {0.7} (S2);

\draw[red, dashed, thick] (7.5, -0.75) ellipse (1cm and 1.8cm);
\draw[red, dashed, thick] (6.625, -1.5) ellipse (1.8cm and 1cm);
\draw[red, dashed, thick] (3, -3.25) ellipse (1cm and 1.8cm);

\draw[->, thick] (S4) edge [loop right] node[right, blue] {$\tau_0$} (S4);
\draw[->, thick] (S5) edge [loop right] node[right, blue] {$\tau_0$} (S5);


\end{tikzpicture}
\caption{The constructed product MDPST $\hat{\M}_{sub}^\times$.} \label{fig:mdpst_v2}
\end{figure}
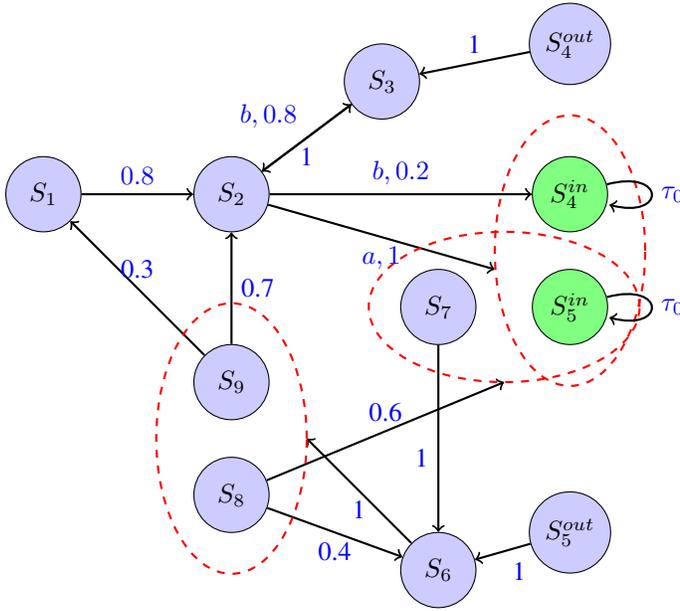

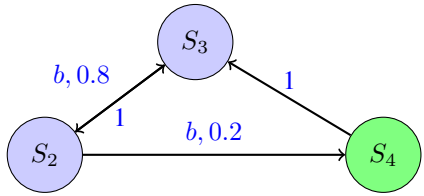
\begin{figure}[t]
\centering
\begin{tikzpicture}

\node[circle, draw, fill=blue!20, minimum width=1cm, minimum height=0.5cm] (S2) at (3, 0) {$S_2$};
\node[circle, draw, fill=blue!20, minimum width=1cm, minimum height=0.5cm] (S3) at (5, 1.5) {$S_3$};
\node[circle, draw, fill=green!50, minimum width=1cm, minimum height=0.5cm] (S4) at (7.5, 0) {$S_4$};

\draw[->, thick] (S2) -- node[midway, above, blue] {$b, 0.2$} (S4); 
\draw[->, thick] (S2) -- node[midway, above left, blue] {$b, 0.8$} (S3);
\draw[->, thick] (S3) -- node[midway, below, blue] {1} (S2); 
\draw[->, thick] (S4) -- node[midway, above, blue] {1} (S3); %
\end{tikzpicture}
\caption{The product MDPST $\M_{sub}^\times$ for the second iteration.}
\label{fig:WR}
\end{figure}

\subsection{Proof of Theorem \ref{thm1}}

We prove Theorem \ref{thm1} in two steps. 

\paragraph{Step 1} We show the correctness of Algorithm 1, i.e., that the output $W^\times$ of Algorithm 1 is indeed the WR of the product MDPST $\M^\times$. More specifically, for every state $s^\times\in W^{\times}$, we have that 
    \begin{itemize}
        \item there exists a strategy $\sigma^\times$ from $s^\times$ such that $\prob_{{\mathcal{M}^\times}}^{\sigma^\times(s^\times)}(\square\diamondsuit \acc^{\times}) = 1$.
    \end{itemize}

    First, for a state $s^\times \in W^{\times}$, $\prob_{{\mathcal{M}^\times}}^{\sigma^\times(s^\times)}(\diamondsuit \acc^{\times}) = 1$, since $\hat{V}_{sat}(s^\times) = 1$, which means that there exists a strategy starting from $s^\times$ and leading to the virtual copies of accepting states in $I_{in}$ with probability one.
    If $s \in W^{\times}$, there must be the virtual copies of accepting states in $I_{out}$ that belong to $W^{\times}$, otherwise those states $s^{out} \in I_{out}$ will be removed from $\acc^{\times}$ (cf. line 10).
    Since $s \in W^{\times}$, we have a strategy $\sigma_1$ to generate a run from $s$ reaching an accepting state, say $s^{in}_1 \in I_{in}$, against any nature $\gamma$.
    Then we know that the virtual copy $s^{out}_1 \in I_{out}$ also belongs to $W^{\times}$, otherwise the accepting state $s_1$ will be removed from $\acc^{\times}$ (cf. line 10).
    As a result, there exists a strategy $\sigma_2$ to generate a run from $s^{out}_1$ (and thus $s_1$) to an accepting state $s^{in}_2\in I_{in}$ against any nature $\gamma$ with probability one.
    Similarly, from $s^{out}_2$, we can extend the run to another accepting state $s^{in}_3$, and so on.
    Since the number of accepting states in $W^{\times}$ is finite, there must be an accepting state that is visited infinitely often.
    Hence, by merging the virtual copies of accepting states into their original state in this generated run, we can obtain an accepting run in $\M^{\times}$.
    Moreover, the probability measure of this run is $1$, as each constructed fragment run has probability measure $1$.
    By combining all these strategies $\sigma_1, \sigma_2, \cdots, \sigma_n$ constructed up to the point where the nature $\gamma$ has repeated its behavior as it only has finite memory, we can obtain the strategy $\sigma = \sigma_1 \cdot \sigma_2 \cdots \sigma_n$.
    It immediately follows that $\prob_{{\mathcal{M}^\times}}^{\sigma^\times(s^\times)}(\square\diamondsuit \acc^{\times}) = 1$.

As long as a run gets trapped in the WR, we have a strategy to keep the run in the WR with probability one.
It immediately follows that:
\begin{equation*}
    \begin{aligned}
        \max\limits_{{\sigma^\times}\in {\Pi_{{\M}^\times}}} \{\prob_{{\M}^\times}^{{{\sigma^\times}}}( \diamondsuit \square W^\times)\} &= \max\limits_{{\sigma^\times}\in {\Pi_{{\M}^\times}}} \{\prob_{{\M}^\times}^{{{\sigma^\times}}}( \diamondsuit W^\times)\}\\
    &= \max\limits_{{\sigma^\times}\in {\Pi_{{\M}^\times}}} \{\prob_{{\M}^\times}^{{{\sigma^\times}}}( \diamondsuit \square \acc^\times)\}.
    \end{aligned}
\end{equation*}

\paragraph{Step 2} We show that
\begin{equation*}
    \max\limits_{\sigma\in \Pi_{\M}}\{\prob_{{\M}}^\sigma(\varphi)\}
        = \max\limits_{{\sigma^\times}\in {\Pi_{{\M}^\times}}} \{\prob_{{\M}^\times}^{{{\sigma^\times}}}( \diamondsuit \square W^\times)\}
\end{equation*}
holds by verifying both sides of the inequality and subsequently constructing the induced policy on $\M$.

Let $p = \max_{\sigma} \{ {\rm Pr}_{{\mathcal{M}}}^\sigma(\varphi) \}$. In other words, there is an \emph{optimal} robust strategy $\sigma$ such that, for every possible nature $\gamma$ of $\mathcal{M}$, we have ${\rm Pr}_\mathcal{M}^{\sigma, \gamma}(\varphi) \geq p$ and there exists some nature $\gamma'$ such that ${\rm Pr}_\mathcal{M}^{\sigma, \gamma'}(\varphi) = p$.
For every fixed strategy $\sigma$ and nature $\gamma$, we can obtain a Markov chain $\mathcal{M}_{\sigma,\gamma}$.

To prove the direction of $\leq$, we will construct a strategy $\sigma^{\times}$ for $\mathcal{M}^{\times}$ such that, for every nature $\gamma^{\times}$, ${\prob^{\sigma^{\times}, \gamma^{\times}}_{\mathcal{M}^{\times}}(\diamondsuit \square W^{\times})} \geq p$.
Observe that every nature $\gamma^{\times}$ for $\mathcal{M}^{\times}$ induces a nature $\gamma$ for $\mathcal{\M}$ by ignoring the component from the automaton $\mathcal{\A}$.
Both Markov chains $\mathcal{M}_{\sigma, \gamma}$ and $\mathcal{M^{\times}_{\sigma^{\times}, \gamma^{\times}}}$ (ignoring the acceptance condition) will have the same probability to generate traces that satisfy $\varphi$, because the probability in $\mathcal{M}^{\times}$ comes only from $\mathcal{M}$.
Thus, for every nature $\gamma^{\times}$ we construct $\sigma^{\times}$ by following $\sigma$ of $\mathcal{M}$ within the state space $S \times Q_i$ and $S \times Q_{acc}$, where $Q_{i}$ (respectively, $Q_{acc}$) is the deterministic part of $Q$ before (respectively, after) seeing an $\epsilon$-transition.
This is possible because the automaton transitions within $Q_i$ and $Q_{acc}$ separately are deterministic.
To resolve the nondeterministic jump from $Q_i$ to $Q_{acc}$ (i.e., $\epsilon$-transitions), we select the automaton successors according to the behavior of the Markov chain $\M_{\sigma, \gamma}$.
That is, before $\M_{\sigma, \gamma}$ enters a bottom strongly connected component (BSCC), $\sigma^{\times}$ follows $\sigma$ in $\M$ and stays within the $Q_i$ part in $\aut$. 
The moment when $\M_{\sigma, \gamma}$ enters a BSCC from a product state $(s, q)$, according to \cite{sickert2016limit}, there is a way to select a successor $q'$ in $Q_{acc}$ for the state $q$ such that the trace generated by $\M_{\sigma, \gamma}$ from $(s, q)$ satisfies $\varphi$ if, and only if, the corresponding run from state $q'$ is accepted by $\aut$.
So, from state $(s, q)$, $\sigma^{\times}$ selects the action $\epsilon_{q'}$ and the run moves to $(s, q')$.
Afterwards, $\sigma^{\times}$ follows $\sigma$ in $\M$ and the run in $\aut$ is again deterministic.
We can see that $\sigma^{\times}$ basically imitates the behavior of $\sigma$ except for resolving the nondeterminism in the automaton $\aut$.
One can see that every trace $\rho$ in $\M_{\sigma, \gamma}$ corresponds to a trace $\rho^{\times}$ in $\M^{\times}_{\sigma^{\times}, \gamma^{\times}}$ with equal transition probabilities except for the $\epsilon$-transition, which has probability $1$.
Moreover, as aforementioned, if $\rho$ satisfies $\varphi$, the run $\rho^{\times}$, projected to the second component, is also an accepting run in $\aut$.

Now we only need to prove that the accepting trace $\rho^{\times}$ will stay within the WR $W^\times$ in the end.
Obviously, all states that occur in $\rho^{\times}$ infinitely many times, denoted $C$, can reach each other.
It must be the case that $C \subseteq W^{\times}$.
This is because the run $\rho^{\times}$ already enters a BSCC of $\M_{\sigma, \gamma}$ and the probability measure of this run in the BSCC is $1$.
Since $\rho$ is accepting, there must be an accepting state $s^{\times} \in \acc^{\times}$ that can visit itself with probability one under the strategy $\sigma$ against any adversarial nature $\gamma$.
Moreover, $s^{\times}$ clearly can be reached from the initial state.
Therefore, after applying Algorithm~1, $s^{\times}$ must belong to $W^{\times}$.
Since all states in the BSCC $C$ can reach accepting state $s^{\times}$ with probability one against any nature $\gamma$, we have that $C \subseteq W^{\times}$.
That is, the accepting run $\rho^{\times}$ eventually stays in $W^{\times}$. We then have that ${\prob^{\sigma^{\times}, \gamma^{\times}}_{\M^{\times}}(\diamondsuit \square W^{\times}) \geq \prob^{\sigma, \gamma}_{\mathcal{M}}(\varphi)} \geq p$ as $\sigma$ is an optimal robust strategy for $\M$.

The direction of $\geq$ is trivial. For any run that eventually stays within the WR $W^{\times}$, it will almost surely visit an accepting state infinitely many times.
Hence, the run must meet the B\"uchi condition.
It is clear that any strategy $\sigma^\times$ on $\M^\times$ can induce a policy for $\M$ by eliminating the $\epsilon$-transitions. Therefore, any path following $\sigma^\times$ that meets the B\"{u}chi condition will induce
a path of $\M$ that is accepting by $\A$ induced from $\varphi$, where the non-determinism of $\A$ is resolved by $\epsilon$-transitions of $\sigma^\times$, thus
satisfying $\varphi$.

Therefore, we have completed the proof.

\end{document}